%% file: arxiv_wo_style.tex
\documentclass[twoside]{article}

\usepackage[preprint]{aistats2026}

\usepackage[round]{natbib}

\input{math_commands.tex}
\usepackage{algorithm}
\usepackage[noend]{algpseudocode}
\usepackage{amsmath}
\usepackage{amsthm}
\usepackage{amssymb}
\usepackage{xcolor}
\usepackage[hyperfootnotes=false]{hyperref}
\hypersetup{
    colorlinks,
    linkcolor={red!50!black},
    citecolor={blue!50!black},
    urlcolor={blue!80!black}
}
\usepackage[normalem]{ulem}
\usepackage{stmaryrd}

\newtheorem{definition}{Definition}

\newtheorem{proposition}{Proposition}

\usepackage{multirow}
\usepackage{threeparttable}
\usepackage{booktabs}
\usepackage{graphicx}

\def\diag{\operatorname{diag}}

\usepackage{url}

\usepackage[symbol*]{footmisc}
\setfnsymbol{wiley}

\begin{document}

\runningtitle{DyKAF: Dynamical Kronecker Approximation of the Fisher Matrix for Gradient Preconditioning}

\runningauthor{Nikolay Yudin, Ekaterina Grishina, Andrey Veprikov, Alexandr Beznosikov, Maxim Rakhuba}
\algrenewcommand\algorithmicindent{1.0em}
\twocolumn[

\aistatstitle{DyKAF: Dynamical Kronecker Approximation of the Fisher Information Matrix for Gradient Preconditioning}

\aistatsauthor{Nikolay Yudin\footnotemark \And Ekaterina Grishina$^\ast$ \And Andrey Veprikov$^\ast$}

\aistatsaddress{HSE University \\ \texttt{neyudin@edu.hse.ru} \And  HSE University \And Basic Research of Artificial \\ Intelligence Laboratory (BRAIn Lab)}
\aistatsauthor{Alexandr Beznosikov \And Maxim Rakhuba}
\aistatsaddress{Basic Research of Artificial \\ Intelligence Laboratory (BRAIn Lab) \\ Innopolis University \And HSE University}]

\footnotetext{Equal contribution}

\begin{abstract}
Recently, optimizers that explicitly treat weights as matrices, rather than flattened vectors, have demonstrated their effectiveness. 
This perspective naturally leads to structured approximations of the Fisher matrix as preconditioners, where the matrix view induces a Kronecker‑factorized form that enables memory‑efficient representation. 
However, constructing such approximations both efficiently and accurately remains an open challenge, since obtaining the optimal factorization is resource‑intensive and practical methods therefore rely on heuristic design choices.
In this work, we introduce a novel approach that leverages projector‑splitting integrators to construct effective preconditioners. 
Our optimizer, $\textbf{DyKAF}$ ($\textbf{Dy}$namical $\textbf{K}$ronecker $\textbf{A}$pproximation of the $\textbf{F}$isher Matrix), consistently improves the Fisher matrix approximation quality. 
Experiments on large language model pre‑training and fine‑tuning demonstrate that $\textbf{DyKAF}$ outperforms existing optimizers across a range of evaluation metrics.
\end{abstract}

\section{Introduction}
\label{sec:intro}

Efficient optimization is crucial for training deep models, both during pretraining on large datasets and fine-tuning for specific tasks \citep{goodfellow2016deep, team2025kimi, parthasarathy2024ultimate}.
Various optimization methods like classical SGD \citep{rumelhart1986learning} or resent Muon \citep{jordan2024muon} rely primarily on gradient information and do not approximate curvature of the loss function. The widely adopted Adam optimizer \citep{kingma2014adam} takes it into account by applying a diagonal preconditioning that scales each parameter independently. More advanced second-order methods estimate the full curvature of the loss surface through approximations of the Hessian matrix and thereby enable faster convergence \citep{semenov2025benchmarking}.

However, exact computation and storage of the Hessian matrix remains infeasible for models with billions of parameters \citep{hernandez2025apertus}, which motivates ongoing research into scalable and accurate approximations. One effective approach is natural gradient descent \citep{amari1998natural}, which uses the Fisher information matrix (see \Eqref{eq:fisher} in Section \ref{sec:fisher}) as an approximation of the Hessian. 

The empirical Fisher information matrix, often computed independently for each neural network layer, provides a practical \citep{martens2020new, frantar2021m, wu2024improved} and theoretically grounded surrogate for the full Hessian under certain conditions \citep{kunstner2019limitations, semenov2025gradient}.

Nevertheless, even when restricted to layerwise blocks, the Fisher information matrix remains prohibitively large, making both its storage and inversion computationally expensive. To address this, several approximations have been proposed, most notably the Kronecker-Factored Approximate Curvature (K-FAC) method \citep{martens2015optimizing}, as well as more recent related optimizers such as Shampoo \citep{gupta2018shampoo} and SOAP \citep{vyas2024soap}. These approaches treat model parameters as matrices and leverage Kronecker factorization to reduce computational complexity while retaining meaningful curvature information. For a comprehensive overview of the Kronecker-Factored methods, we refer the reader to Section~\ref{sec:related_work}.

A fundamental limitation of Kronecker approximation methods is the inability to compute optimal factors without the full Fisher matrix. Consequently, these methods often inaccurately approximate the true Hessian, as shown in Figure~\ref{fig:hess_diff}. In this small-scale experiment, we evaluate Hessian approximation accuracy in a classification setup, where the Hessian is analytically computable (see details in Appendix~\ref{app:hessian_calc}). This setup allows a direct comparison between the preconditions of state-of-the-art methods like SOAP \citep{vyas2024soap} and our proposed optimizer, DyKAF (Algorithm \ref{alg:dykaf}, Section \ref{sec:optimizer}).

\begin{figure}[h!]
    \centering
    \includegraphics[width=\linewidth]{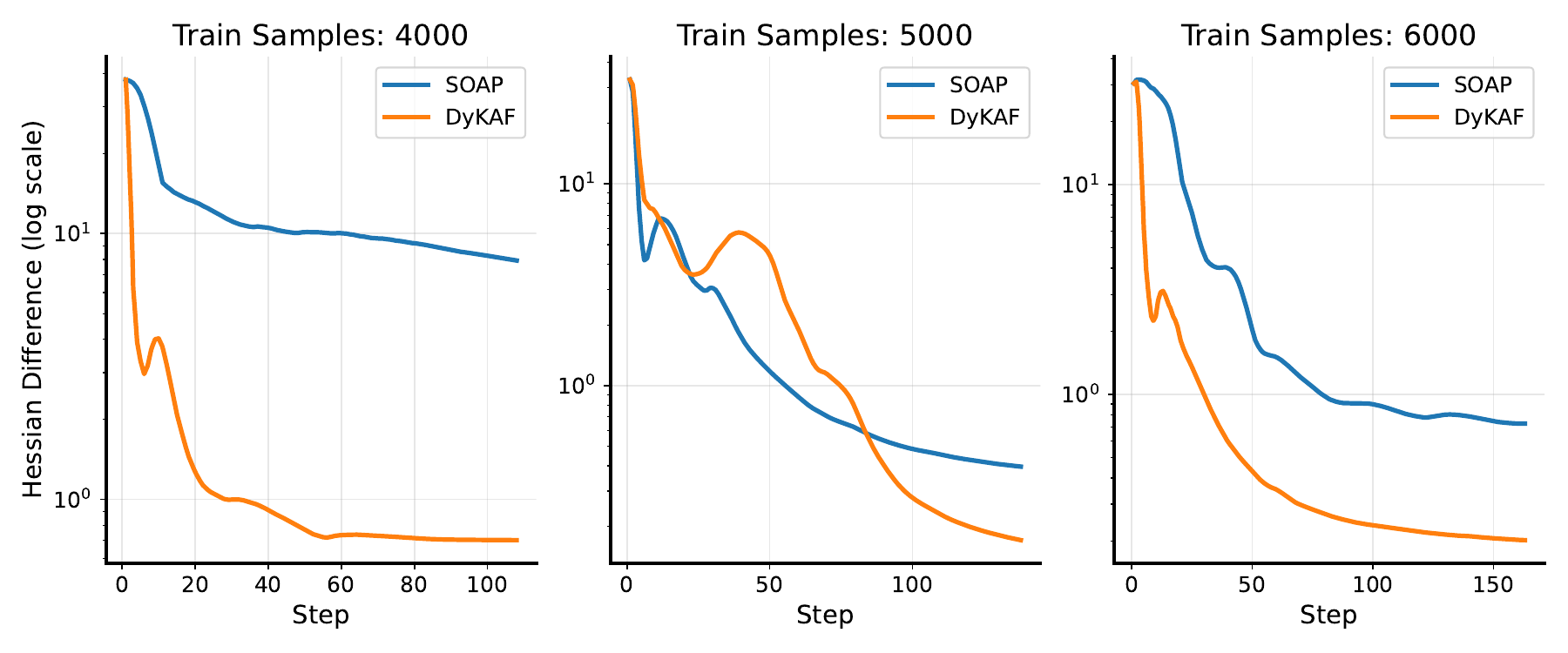}
    \caption{
    Comparison of Hessian approximation in the controlled environment. The Frobenius norm of the difference between the Hessian and its approximations by SOAP and DyKAF (Algorithm \ref{alg:dykaf}) is shown for varying training sample sizes. Our algorithm achieves consistently better approximation accuracy.
    }

    \label{fig:hess_diff}
\end{figure}

An imprecise approximation of the Fisher information matrix in existing methods can limit the effectiveness of gradient preconditioning. This work aims to improve the accuracy of Kronecker product updates via the so-called projector-splitting scheme for dynamical low-rank approximation. Our main contributions are:
\begin{itemize}
\item  We propose an algorithm to maintain the Kronecker product approximation of the Fisher matrix using projector-splitting integrator \citep{lubich2014projector, ceruti2022rank}. 
We incorporate this approach into SOAP optimizer by modifying the update formulas for Kronecker factors 
(see ~\Algref{alg:dykaf}). These updates add negligible overhead (see \Algref{alg:kron_proj_split}) and integrate seamlessly into the SOAP framework, using the same hyperparameters (see Section~\ref{sub:pt}).

\item We provide a theoretical analysis comparing the accuracy of the dynamical Kronecker approximation of Fisher matrix and the update formula of Shampoo and SOAP (see Section \ref{sec:fisher}). Our theoretical findings are supported by empirical evaluations.

\item We conduct a thorough evaluation of our method on language model fine-tuning (Section \ref{sub:ft}) and pretraining (Section \ref{sub:pt}) tasks showing a performance improvement in comparison with the state-of-the-art alternatives. 

\end{itemize}

\section{Notation and Background}
\label{sec:notation}
\subsection{Preliminaries}
In this section, we recap some well-known facts from linear algebra, see \citep{golub2013matrix}. Note that we will use row-major index ordering  as in Pytorch.
Further in the text, $\|A\|=\sqrt{\sum_{ij}A_{ij}^2}$ denotes the Frobenius norm.

\begin{definition}
  For any square matrix $A \in \mathbb{R}^{n \times n}$, let $\mathrm{off}(A) \in \mathbb{R}^{n \times n}$ be the matrix obtained by subtracting the diagonal part of $A$:
  \begin{equation*}
    \mathrm{off}(A) := A - \mathrm{diag}(A). 
  \end{equation*}
\end{definition} 

\begin{definition}
    For any matrix $A \in \mathbb{R}^{m \times n}$, let $A^{\circ \alpha}$ denote a Hadamard (elementwise) power of a matrix:
    \begin{equation*}
        (A^{\circ \alpha})_{ij} = A_{ij}^{\alpha}, \quad \forall~ i, j.
    \end{equation*}
\end{definition}

\begin{definition}
 Let $\mathrm{vec}(X) \in \mathbb{R}^{mn}$ denote a vector obtained by “stacking” rows of the matrix $X \in \mathbb{R}^{m\times n}$:
 \[\mathrm{vec}(X):=\begin{bmatrix}X[0, :]^{\top} \\ \vdots \\ X[n-1, :]^{\top}\end{bmatrix}.\]
\end{definition}
It is well known that for $B \in \mathbb{R}^{m_1\times n_1}, C \in \mathbb{R}^{m_2\times n_2}$, and $X \in \mathbb{R}^{n_1\times n_2}$, the following identity holds:
\begin{equation}\label{eq:kron_product_vectorization}
	Y = BXC^{\top} \Leftrightarrow \mathrm{vec}(Y)=(B \otimes C)\mathrm{vec}(X),
\end{equation}
where $\otimes$ denotes the Kronecker product.
Similarly to the definition of the vectorization operation, let us define an inverse operation called matricization.
\begin{definition}
    For a vector $x \in \mathbb{R}^{mn}$ let the matricization operator $\mathrm{Mat}(x) \in \mathbb{R}^{m \times n}$ denote a matrix which is obtained by ``packing'' $x$ into an $m$-by-$n$ matrix in the row-major order:
    \begin{equation*}
       \mathrm{Mat}(x) :=
       \begin{bmatrix}
		   x[0{:}n]^{\top}\\
            \dots\\
			x[in{:} (i+1)n]^{\top}\\
            \dots\\
			x[(m-1)n{:} mn]^{\top}
       \end{bmatrix}.
    \end{equation*}
\end{definition}
Let  $A \in \mathbb{R}^{m_1 m_2\times n_1 n_2}$. 
The Nearest Kronecker Product (NKP) problem reads as: 
\begin{equation}\label{eq:nkp}
\| A - B \otimes C \|\to\min_{B, C},
\end{equation}
where $B \in \mathbb{R}^{m_1\times n_1}$ and $ C \in \mathbb{R}^{m_2\times n_2}$. This problem can be solved using a rearranged version of $A$.
\begin{definition}Let us define a rearrangement operator $\mathcal{R}(\cdot)$, which rearranges elements from $A \in \mathbb{R}^{m_1m_2\times n_1n_2}$ into $\mathcal{R}(A) \in \mathbb{R}^{m_1n_1\times m_2n_2}$ in the following way:
\[\mathcal{R}(A)[m_2 i+i', n_2j+j']=A[n_1i+j, n_2i'+j'].\]
\end{definition}
The rearrangement operator has several useful properties. In particular, it maps Kronecker product into a rank-1 matrix:
\begin{equation*}
	\mathcal{R}(B\otimes C)=\text{vec}(B)\text{vec}(C)^{\top}.
\end{equation*}
This property of the rearrangement operator helps to obtain the following fact.
\begin{proposition}[\cite{golub2013matrix}]\label{prop:rearrangement_golub}
The NKP approximation problem~\ref{eq:nkp} is equivalent to finding the best rank-1 approximation to $\mathcal{R}(A)$:
\begin{equation*}
	\begin{split}
	 \|\mathcal{R}(A) -\mathrm{vec}(B) \mathrm{vec}(C)^{\top}\| \to \min_{B, C}.
	\end{split}
\end{equation*}

\end{proposition}

Thus, the proposition above allows us to approach the NKP problem for a certain matrix $A \in \mathbb{R}^{mn \times mn}$ as a rank-1 approximation of the rearranged matrix $\mathcal{R}(A) \in \mathbb{R}^{m^2 \times n^2}$.

\section{The Fisher Information Matrix Angle of Shampoo and SOAP}
Let us first define the empirical Fisher information matrix $F_t \in \mathbb{R}^{mn\times mn}$ for a layer with a weight matrix $W \in \mathbb{R}^{m \times n}$ as
\begin{equation}
    \label{eq:fisher}
	F_t := \sum_{i=1}^t \mathrm{vec}(G_i)\mathrm{vec}(G_i)^{\top},
\end{equation}
where $G_i \in \mathbb{R}^{m \times n}$ is the gradient with respect to the layer $W$ on the $i$-th optimization step. Note that in this work we consider only matrix-shaped parameters.

Recent works \citep{martens2015optimizing, gupta2018shampoo, vyas2024soap} have investigated methods to approximate the Fisher matrix with a Kronecker product of two smaller matrices $L_t \in \mathbb{R}^{m \times m}, R_t \in \mathbb{R}^{n \times n}$.
In particular, the authors of Shampoo \citep{gupta2018shampoo} proposed to
update Kronecker factors using simple formulas based on the current gradient $G_{t} \in \mathbb{R}^{m \times n}$ of the layer:
\begin{equation}\label{eq:shampoo_update_rule}
    L_{t+1} = L_{t} + G_tG_t^{\top}, \quad R_{t+1} = R_{t} + G_{t}^{\top}G_{t}.
\end{equation}
Unfolding these recurrent formulas and starting with $L_0=\varepsilon I, R_0=\varepsilon I$, \citep{gupta2018shampoo} aims to substitute the exact empirical Fisher matrix with its theoretically-proven upper bound in the Kronecker product form.
Under assumption that every $G_i$ has rank at most $r$, the authors obtain the following upper bound: 
\begin{equation*}
\begin{split}
	\varepsilon I_{mn} + &\frac{1}{r} F_t \preceq \\ \preceq & \left(\varepsilon I + \sum_i^t G_iG_i^{\top}\right)^{\frac{1}{2}} \otimes \left(\varepsilon I + \sum_i^t G_i^{\top}G_i\right)^{\frac{1}{2}},
\end{split}
\end{equation*}
where $A \preceq B$ implies that $B - A$ is a symmetric positive definite matrix. As a result, Shampoo's preconditioned gradient $G_t'$ can be derived using the following formula:

\begin{equation*}
	\begin{split}
		\mathrm{vec}(G_t') &= F_t^{-\frac{1}{2}}\mathrm{vec}(G_t) \approx (L_t^{-\frac{1}{4}} \otimes R_t^{-\frac{1}{4}})\mathrm{vec}(G_t) =\\&= \mathrm{vec}(L_t^{-\frac{1}{4}}G_tR_t^{-\frac{1}{4}}).
	\end{split}
\end{equation*}

However, the upper bound defined in ~\Eqref{eq:shampoo_update_rule} may be far from $F_t$ (see Figure~\ref{fig:random}), and, what is more, taking matrix inverse root of nearly singular matrices

on each optimization step leads to numerical instabilities and is computationally demanding.

SOAP \citep{vyas2024soap}  avoids the instability issue by rotating the gradient to align it with the eigenspace of the Fisher matrix and use the diagonal part for preconditioning. 
The idea behind this approach is that if we are given eigenvalue decompositions of $L_t = Q_{L}\Lambda_LQ_{L}^{\top}$ and $R_t = Q_{R}\Lambda_R Q_{R}^{\top}$, the eigenspace of the Fisher matrix can be approximated using the Kronecker product $Q_L\otimes Q_R$. 
Indeed, if $\tilde{F_t} \approx L_t \otimes R_t$ then we can rewrite $\tilde{F_t}$ as follows:
\begin{equation*}
    \begin{split}
    \tilde{F_t} \approx L_t \otimes R_t = &(Q_L \Lambda_L Q_L^{\top}) \otimes (Q_R \Lambda_R Q_R^{\top}) =  \\ = & (Q_L \otimes Q_R)^{\top}(\Lambda_L \otimes \Lambda_R)(Q_L \otimes Q_R).
    \end{split}
\end{equation*}
In the eigenbasis of the Kronecker product approximation, the Fisher matrix is expected to be  closer to diagonal (we obtain a formal statement in Section~\ref{sec:optimizer}). As a result, one can replace it with a diagonal matrix in this basis, which is more accurate than for the original matrix without rotations, like in Adam~\citep{kingma2014adam}.
Let $\mathrm{vec}(V_t)$ be the diagonal part of the rotated Fisher matrix. Then, we have: 

\begin{equation*}
	\begin{split}
		& \mathrm{vec}(G'_t)=F_t^{-\frac{1}{2}}\mathrm{vec}(G_t)\approx\\&\approx (Q_L\otimes Q_R)\diag(\mathrm{vec}(V_t))^{-\frac{1}{2}}(Q_L\otimes Q_R)^{\top}\mathrm{vec}(G_t)=\\&=(Q_L\otimes Q_R)\diag(\mathrm{vec}(V_t^{\circ -\frac{1}{2}}))\mathrm{vec}(Q_L^{\top}G_tQ_R) = \\ & = ~\mathrm{vec}(Q_L((Q_L^{\top}G_tQ_R) \oslash V_t^{\circ \frac{1}{2}})Q_R^{\top}),
	\end{split}
\end{equation*}
where $\oslash$ denotes elementwise division.

Although \citep{vyas2024soap} introduces a different approach compared to Shampoo, they can be unified from the matrix analysis perspective.
Indeed, Shampoo preconditioning
can be rewritten as follows:
\begin{equation*}
  \begin{split}
	  \mathrm{vec}(G'_t) &=  (L^{-\frac 14}_t \otimes R^{-\frac 14}_t)\mathrm{vec}(G_t)  \\& = (Q_L\Lambda_L^{-\frac 14}Q_L^{\top}) \otimes (Q_R\Lambda_R^{-\frac 14}Q_R^{\top})\mathrm{vec}(G_t)=  
	  \\ & = (Q_L \otimes Q_R) (\Lambda_L \otimes \Lambda_R)^{-\frac 14} \mathrm{vec}(Q_L^{\top}G_tQ_R) = \\ & =  \mathrm{vec}(Q_L((Q_L^{\top}G_tQ_R) \oslash (\lambda_L\lambda_R^{\top})^{\circ -\frac{1}{4}})Q_R^{\top}).
  \end{split}
\end{equation*}

The expression above shows that Shampoo optimizer also rotates the gradient and applies elementwise division in the rotated space, but with an additional constraint for the diagonal.

We argue that maintaining a more accurate approximation of $L_t$ and $R_t$ and, consequently, of $Q_L$ and $Q_R$ may improve the approximation of the Fisher matrix eigenspace in which we rotate the gradient and make the rotated matrix ``more diagonal'' in terms of the Frobenius norm.

In order to improve the quality of the approximation, we seek to use dynamic low-rank approximation approach, which recently showed its efficiency in full-matrix preconditioning methods in the case of dynamic Cholesky decomposition \citep{matveeva2025dynamic}.

\section{Dynamic Low-rank Approximation of the Fisher Information Matrix}
\label{sec:fisher}

Building on the definitions from Section \ref{sec:notation} and their properties, in this work we aim to efficiently and accurately maintain the Kronecker product approximation to the empirical Fisher matrix over the optimization steps. Proposition \ref{prop:rearrangement_golub} tells us that the approximation of $F_t$ via a Kronecker product can be treated as a rank-1 approximation of the rearranged matrix:
\begin{equation*}
	\|F_t-L_t\otimes R_t\|=\|\mathcal{R}(F_{t}) - \mathrm{vec}(L_t)\mathrm{vec}(R_t)^{\top}\|\to \min_{L_t, R_t}
\end{equation*}

In \citep{lubich2014projector, ceruti2022rank}, the authors introduce projector-splitting integrators to maintain an accurate low-rank approximation of the time-dependent matrix trajectory.
This approach can also be applied when the matrix undergoes discrete time updates. 
In Algorithm  \ref{alg:proj_split}, we adapt the projector-splitting algorithm provided in \citep{ceruti2022rank} for the discrete case. Here and further we will denote a general version of the algorithm as \texttt{proj\_split}.
\begin{algorithm}[h!]
	\caption{\texttt{proj\_split}. Dynamical low-rank approximation with discrete update \citep{ceruti2022rank}}\label{alg:proj_split}
\begin{algorithmic}[1]
	\Require  $\tilde{F_t} = U_tS_tV_t^{\top}$ -- current low-rank approximation to $F_t$ ($U \in \mathbb{R}^{m\times r}$, $V \in \mathbb{R}^{n \times r}$ have orthonormal columns, $S \in \mathbb{R}^{r \times r}$ is invertible but not necessary diagonal); $\Delta F_t \in \mathbb{R}^{m \times n}$ -- current update of $F_t$ at the step $t$.
\State $\hat{U_t} = U_tS_t + \Delta F_t V_t$;
\State $U_{t+1}\hat{S_{t+1}} = QR(\hat{U_t})$;
\State $\hat{V_t} = V_tS_t^{\top} + \Delta F_t^{\top} U_t$;
\State $V_{t+1}\tilde{S_{t+1}} = QR(\hat{V_{t}})$;
\State $S_{t+1} = U_{t+1}^{\top} (\tilde{F_t} + \Delta F_t) V_{t+1}$;\\
\Return $\tilde{F}_{t+1} = U_{t+1} S_{t+1} V_{t+1}^{\top} \approx F_{t+1}$.
\end{algorithmic}	
\end{algorithm}

Note that \Algref{alg:proj_split} depends only on the current low-rank approximation $\tilde{F_t}$ and the current update $\Delta F_t$. Moreover, the provided algorithm in the case of dynamical rank-1 approximation, which is our case, does not require QR decompositions at all, as they are equivalent to vector normalization.

Now we can adapt Algorithm \ref{alg:proj_split} for our purposes. Here and further we will denote it as \texttt{kron\_proj\_split}. Following the \Eqref{eq:fisher}, after the rearrangement the update of the Fisher matrix  has a particular Kronecker product structure:
\begin{equation}\label{eq:fisher_update} \mathcal{R}(F_t)=\mathcal{R}(F_{t-1})+G_{t}\otimes G_t^{\top},\end{equation}
implying that $\Delta F_t = G_{t}\otimes G_t^{\top}$.
Using \Eqref{eq:kron_product_vectorization}, we observe that the projector-splitting algorithm has an efficient implementation, which does not require matrix-vector products with dense matrices of $mn \times mn$ size.
In \Algref{alg:kron_proj_split}, we detail a computationally-efficient implementation of the dynamical low-rank approximation from Algorithm~\ref{alg:proj_split}  for the matrix $F_t$ defined in \Eqref{eq:fisher}. A generalization for tensors is presented in Algorithm \ref{alg:kron_proj_split_tensor} (Appendix \ref{apx:tensors}).

\begin{algorithm}
  \caption{ 
  \texttt{kron\_proj\_split}. Dynamical Kronecker approximation of Fisher matrix.
  }\label{alg:kron_proj_split}
\begin{algorithmic}[1]
	\Require 
    $L_t \in \mathbb{R}^{m \times m}, R_t \in \mathbb{R}^{n \times n}$ -- Kronecker product approximation of the Fisher matrix on the optimization step $t$, $G_t \in \mathbb{R}^{m \times n}$ -- gradient with respect to the layer $W_t$.
	\State $\hat{L_t} = \displaystyle L_t \left\|{R_t}\right\|_{{F}} + G_t (R_t/\left\|{R_t}\right\|_{{F}})G_t^{\top}$
	\State $\hat{R_t} = \displaystyle R_t \left\|{L_t}\right\|_{{F}} + G_t^{\top}(L_t/\left\|{L_t}\right\|_{{F}})G_t$
	\State $L_{t+1} = \hat{L_t}/  \|{\hat{L_t}}\|_{{F}}$
	\State $R_{t+1} = \hat{R_t} / \|{\hat{R_t}}\|_{{F}}$
	\State $S = \displaystyle \langle {L_t}, {L_{t+1}}\rangle_{{F}}\langle {R_t}, {R_{t+1}}\rangle_{{F}} + \langle {L_{t+1}}, {G_tR_{t+1}G_t^{\top}}\rangle_{{F}}$
  \State \Return $S^{1/2}L_{t+1}, S^{1/2}R_{t+1}$
\end{algorithmic}
\end{algorithm}

This algorithm can be readily inserted in optimizers that rely upon the Kronecker product structure such as SOAP and Shampoo  without noticeable memory and computational overhead.
Moreover, we provide a theoretical explanation of the validity of such approach in Proposition \ref{thm:thm_dynamical} which shows that if the Fisher matrix can be accurately approximated with the help of the Kronecker product during training (see \citep{martens2015optimizing} for motivation), the projector-splitting is able to capture the Kronecker product structure on each step.
Let us consider one step of the integrator.

\begin{proposition} \label{thm:thm_dynamical}

Let $F_0$ and $F_1$ be such that:
\begin{equation*}
	F_{0}=A_{0}\otimes B_{0}+E_{0},
\end{equation*}
\begin{equation*}
	F_{1}=A_{1}\otimes B_{1}+E_{1},
\end{equation*}
where $\|E_i\|\leq \varepsilon \|F_i\|$ and
$\frac{|tr((A_0\otimes B_0)^\top (A_1\otimes B_1))|}{\|A_0\otimes B_0\|\|A_1\otimes B_1\|}\geq c>0.$
Then, one step of the dynamical low-rank approximation from \Algref{alg:kron_proj_split} starting with $L_0, R_0=A_0, B_0$ yields the approximation that of $F_1$ such that:
\[L_1\otimes R_1=A_1\otimes B_1+ \widehat E_1+O(\varepsilon^2),\]
where $\| \widehat E_1\|\leq \left(1+\frac{2}{c}\right) (\|E_1\|+\|E_0\|)$.
\end{proposition}

\begin{proof}See Appendix \ref{apx:thm_dynamical}.
\end{proof}
 
On the contrary, the accuracy of Shampoo approximation depends on the degree of collinearity of the gradients (mutual coherence $\mu$). 
\begin{definition} The mutual coherence $\mu$ of a set of vectors $\{g_i\}_{i=1}^{n}$ is defined as: 
\begin{equation*}
	\mu=\max_{i\not=j}\frac{|g_i^{\top}g_j|}{\|g_i\|\|g_j\|}.    
\end{equation*}

Mutual coherence $\mu$ ranges from 0 for pairwise orthogonal vectors (if $n$ is less than size of $g_i$) to 1 for collinear.
\end{definition}

The following proposition shows that the less aligned the vectors $\mathrm{vec}(G_i)$ are, the more is the difference in the norms of Fisher matrix and Shampoo estimate.
\begin{proposition}\label{thm:thm_coherence} Let $\mu_t$ be the mutual coherence of the gradients $\{\mathrm{vec}(G_i)\}_{i=1}^t$. 
	Let $L_t=\sum_i^tG_iG_i^{\top}, R_t=\sum_i^tG_i^{\top}G_i$ (as in Shampoo and SOAP). Then
\[\left\|L_{t}^{\frac{1}{2}}\otimes R_{t}^{\frac{1}{2}}\right\|^2-\left\|F_t\right\|^2\geq (1-\mu_t^2)  \sum_{i}^t\sum_{j\not=i}^t\|G_i\|^2\|G_j\|^2.\]
\end{proposition}
\begin{proof}See Appendix \ref{apx:thm_coherence}.
\end{proof}

If the gradients are close to collinear, then $\mu_t\approx 1$ and the norm of Shampoo approximation is close to the true Fisher norm. However, when the directions of the gradients vary, which may occur in practice, the Shampoo approximation becomes less accurate.

To also empirically support our findings, we sample $G_i\in \mathbb{R}^{32\times32}$ from $\mathcal{N}(0, I)$ and sum them with exponential moving average $\beta=0.9$, obtaining the Fisher matrix, following Shampoo optimizer \citep{gupta2018shampoo} setup.
Figure \ref{fig:random} demonstrates the error Fisher matrix approximation with different methods. The error of DyKAF approximation is almost the same as of the best Kronecker product approximation, which means that it is significantly closer to true Fisher matrix, than Shampoo.
\begin{figure}[h!]
    \centering
    \includegraphics[width=1\linewidth]{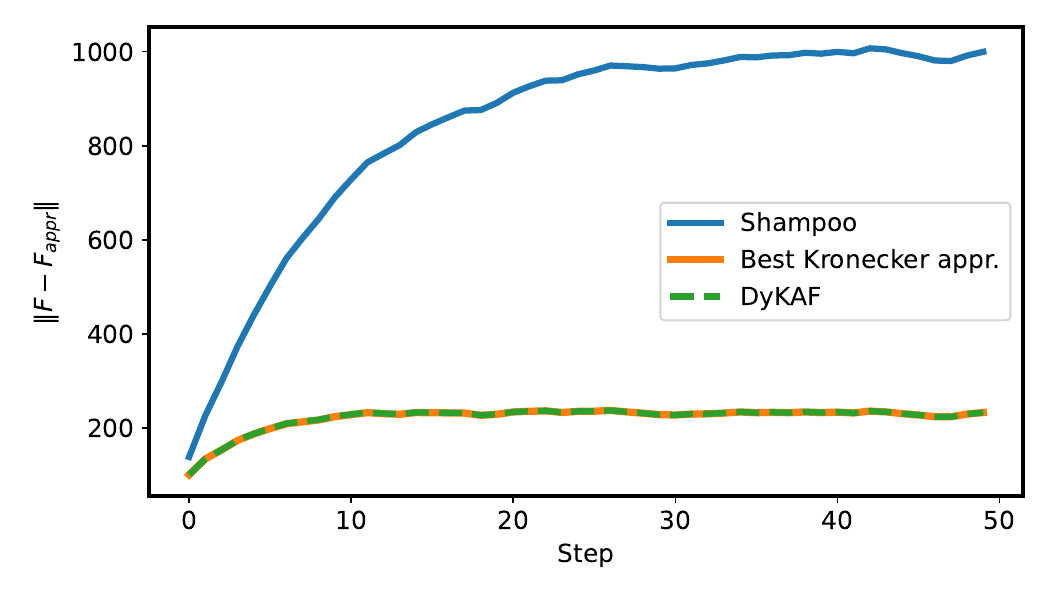}
    \caption{Fisher matrix approximation error with different methods in the case when $G_i$ are sampled from $\mathcal{N}(0, I)$.}
    \label{fig:random}
\end{figure}

\section{DyKAF Optimizer}
\label{sec:optimizer}
We are now ready to define our modification of the SOAP optimizer, which is the core contribution of our work.
DyKAF optimizer (Algorithm \ref{alg:dykaf}) incorporates a dynamical Kronecker approximation of the Fisher information matrix using Algorithm \ref{alg:kron_proj_split}.
Before describing the DyKAF optimizer, let us start with the initialization of Kronecker factors in Algorithm \ref{alg:dykaf}.
\subsection{Initialization}
Unlike Shampoo and SOAP, which initialize Kronecker factors with scaled identity matrices, we start with the best Kronecker approximation of the Fisher matrix on the first step $F_1=\mathrm{vec}(G_1) \mathrm{vec}(G_1)^{\top}$. 
This initialization can be efficiently obtained using the power method for $G_1$, 
as it stated in Proposition~\ref{lemma:init}.
\begin{proposition}\label{lemma:init}
The solution to
\[
\min_{L, R} \|\mathrm{vec}(G_1) \mathrm{vec}(G_1)^\top - L \otimes R\|
\]
is given by \(L = \sigma_1(G_1) u_1 u_1^\top\) and \(R = \sigma_1(G_1) v_1 v_1^\top\), where \(u_1\) and \(v_1\) are the left and the right singular vectors of \(G_1\) corresponding to the largest singular value \(\sigma_1(G_1)\).
\end{proposition}
\begin{proof}
See Appendix~\ref{apx:proof_init}.
\end{proof}

\subsection{Adam Second Moment in the Rotated Space}

We additionally investigate the estimation of the Fisher matrix eigenvalues. SOAP approximates them using
the second moment $V_t \in \mathbb{R}^{m \times n}$ in the rotated space as follows:
\begin{equation*}
    V_t = \beta_2 V_{t-1} + (1-\beta_2) (G_t' \odot G_t').
\end{equation*}
In this work, we introduce an alternative way to estimate eigenvalues of the Fisher matrix.

Let us assume that $F$ has the following form:
\begin{equation*}
    F_t = L_t \otimes R_t + E_t,
\end{equation*}
where $L_t$ and $R_t$ are the current Kronecker product factors.
If we know eigenvalue decompositions $L_t = Q_L\Lambda_L Q_L$ and $R_t = Q_R\Lambda_R Q_R^{\top}$, then we have
\begin{align}
    \tilde{F_t} &= (Q_L \otimes Q_R)^{\top}F_t(Q_L \otimes Q_R) + E_t\nonumber\\&= \Lambda_L \otimes \Lambda_R + \tilde{E_t}, \label{eq:rotated_fisher_form}
\end{align}
where $\tilde{E_t} = (Q_L \otimes Q_R)^{\top}E_t(Q_L \otimes Q_R)$.
If we work under assumption that the Fisher matrix can be accurately approximated as a Kronecker product, then $\tilde{F_t}$ is closer to a diagonal matrix than $F_t$.
We provide a more formal version in the following proposition.
\begin{proposition}\label{thm:fisher_diag}
Let $F \in \mathbb{R}^{mn \times mn}$ be a symmetric matrix which can be represented in the following form:
  \begin{equation}\label{eq:weak_conditions_th1}
  	F_t = A_t \otimes B_t + E_t,
  \end{equation}
  where $E_t \in \mathbb{R}^{mn \times mn}$ satisfies $\|E_t\| \leqslant \left\|{\mathrm{off}(F)}\right\|_{{F_t}}$, and $A_t \in \mathbb{R}^{m \times m}, B_t \in \mathbb{R}^{n \times n}$ are symmetric matrices. Let $A_t = Q_A\Lambda_A Q_A^{\top}$, $B_t = Q_B\Lambda_B Q_B^{\top}$ be the eigenvalue decompositions. Then
  \begin{equation*}
	  \left\|{\mathrm{diag}\left((Q_A \otimes Q_B)^{\top}F_t(Q_A \otimes Q_B)\right)}\right\|_{{F}} \geqslant \left\|{\mathrm{diag}(F_t)}\right\|_{{F}}.
  \end{equation*}
  and
  \begin{equation*}
	  \left\|{\mathrm{off}\left((Q_A \otimes Q_B)^{\top}F_t(Q_A \otimes Q_B)\right)}\right\|_{{F}} \leqslant \left\|{\mathrm{off}(F_t)}\right\|_{{F}}.
  \end{equation*}
\end{proposition}
\begin{proof}
	See Appendix \ref{app:zlp}.
\end{proof}

The proposition states that if we have a accurate Kronecker product approximation of the Fisher matrix \citep{martens2015optimizing}, then the rotated matrix $\widetilde{F_t}$ is ``more diagonal'' than $F_t$ in terms of the Frobenius norm. 
It additionally highlights the importance of having a good approximation of the $Q_L$ and $Q_R$, as these matrices explicitly affect the proximity of the rotated matrix to the diagonal matrix.

The assumption that the rotated Fisher matrix is close to a diagonal matrix makes it also possible to try a lightweight approximation and efficiently invert it.

At the same time, under the assumptions of Proposition \ref{thm:fisher_diag} the diagonal part of the rotated matrix $\tilde{F_t}$ can be written in the form
\begin{equation*}
    \diag(\tilde{F_t}) = \Lambda_L \otimes \Lambda_R + \diag(\tilde{E_t}),
\end{equation*}
which, in the case of small $\|\tilde{E_t}\|$, can be well approximated by Kronecker product of two diagonal matrices.
Hence, matricization of the diagonal of rotated Fisher matrix $\tilde{F_t}$ can be approximated via rank-1 matrix under assumptions in Proposition \ref{thm:fisher_diag}.
This allows us to apply the same projector-splitting approach to approximation of the Adam's second moment
, additionally decreasing optimizer's memory consumption.

\subsection{DyKAF Optimizer}

Now we are ready to formulate the main algorithm of this work. Algorithm \ref{alg:dykaf} follows the SOAP framework and can be viewed as its modification that uses projector-splitting to keep an accurate dynamical Kronecker approximation, while remaining memory- and time-efficient (see Appendix \ref{sec:perf_eval}). It  integrates cleanly into existing SOAP code, and, as shown in experiments, working well with the same hyperparameters as SOAP (see Section \ref{sub:pt}).

Regarding the second moment maintenance, DyKAF supports two modes controlled by the flag \texttt{rank1\_second\_moment}. If enabled, the Adam second moment matrix \(V_t \in \mathbb{R}^{m\times n}\) is approximated with product of two vectors $v_l \in \mathbb{R}^{m}$ and $v_r \in \mathbb{R}^{n}$, which are initialized to small constants \(\varepsilon \mathbf{1}\). Otherwise, the standard AdamW is used with initialization \(V_0 = 0\). We recommend setting \texttt{rank1\_second\_moment = True} for fine-tuning and \texttt{False} for pretraining (see Appendix \ref{app:ablations_llama}). The momentum term \(M_0\) is initialized to zero.

\begin{algorithm}[H]
	\begin{algorithmic}[1]
		\caption{\textbf{DyKAF} optimizer's 
        step $t$ for a $m \times n$ layer $W_t$.
        \textit{Hyperparameters}: learning rate~$\eta$, betas $(\beta_1, \beta_2)$, epsilon $\varepsilon$, preconditioning frequency~$f$ and Adam rank-1 flag \texttt{rank1\_second\_moment}.}
        \label{alg:dykaf}

\State Sample batch $B_t$
\State $G_t \in \mathbb{R}^{m \times n} \gets -\nabla_W \mathcal{L}_{B_t}(W_t)$.
\State $G_t' \gets Q_L^{\top} G_t Q_R$
\State $M_t \gets \beta_1 M_t + (1-\beta_1) G_t$
\State $M'_t \gets Q_L^{\top} M_t Q_R$
\If{\texttt{rank1\_second\_moment}}
\comment{Apply alg.~\ref{alg:proj_split} to update rank-1 Adam second moment approximation and use in step 10}{\algorithmicindent}
\State $v_l v_r^{\top} \gets \texttt{proj\_split}(v_{l}v_{r}^{\top}, G'_t \odot G'_t)$
\State $N'_t \gets M'_t \oslash \left(\left(v_lv_r^{\top}\right)^{\circ \frac{1}{2}}+\epsilon \right)$
\Else
\State $V_t \gets \beta_2 V_t + (1-\beta_2) (G'_t \odot G'_t)$
\State $N'_t \gets M'_t \oslash \left(V_t^{\circ \frac{1}{2}}+\epsilon \right)$
\EndIf
\State $N_t \gets Q_L N'_t Q_R^{\top}$ 
\State $W_t \gets W_t - \eta N_t$
\comment{Apply alg.~\ref{alg:kron_proj_split} to update $L$ and $R$}{0em}
\State $L, R \gets \texttt{kron\_proj\_split}(\sqrt{\beta_1}L, \sqrt{\beta_1}R, \sqrt{1-\beta_1} G_t)$

\If{$t ~\%~ f = 0$}
\comment{Apply alg.~\ref{alg:eigenvectors} to update $Q_L$ and $Q_R$.}{\algorithmicindent}
\State $Q_L \gets \texttt{Eigenvectors}(L, Q_L)$
\State $Q_R \gets \texttt{Eigenvectors}(R, Q_R)$
\EndIf
  \end{algorithmic}
\end{algorithm}

\begin{algorithm}
	\begin{algorithmic}[1]
		\State $S \gets PQ$
            \State $Q \gets \texttt{QR}(S)$
	\end{algorithmic}
	\caption{\texttt{Eigenvectors} function from \citep{vyas2024soap}. }
	\label{alg:eigenvectors}
\end{algorithm}

\section{Experiments}
We organize the experimental evaluation into two stages.

First we focus on fine-tuning benchmarks (Section \ref{sub:ft}) and  after we move to pretraining tasks (Section \ref{sub:pt}). We compare DyKAF (Algorithm \ref{alg:dykaf}) against AdamW \citep{kingma2014adam}, Muon~\citep{jordan2024muon}, and SOAP \citep{vyas2024soap}.
These optimizers constitute a representative set of state-of-the-art methods, combining strong empirical performance with theoretical appeal, and thus form a compelling baseline suite for comparison. 
\subsection{Fine-tuning}
\label{sub:ft}
We evaluate DyKAF (Algorithm \ref{alg:dykaf}) in two fine-tuning scenarios covering both medium-scale and large-scale models.
In all runs, only the self-attention matrices 
are updated while the rest of the model remains frozen.
For DyKAF optimizer we set \texttt{rank1\_second\_moment} parameter to \texttt{True}, which is motivated by the empirical observation that the gradients in fine-tuning exhibit low-rank structure (see ablation study in Appendix \ref{app:ablations_llama}).
Other hyperparameter details are provided in Appendix~\ref{app:ft}.

\paragraph{Sentence-level language understanding.} 
Our first setup considers \texttt{Roberta-Large} model \citep{he2021debertav3} (185M parameters) on the GLUE benchmark \citep{wang2018glue}. We adopt full fine-tuning of the self-attention matrices, i.e., without any parameter-efficient adapters. This represents a controlled medium-scale setting where optimization dynamics can be clearly compared. The results are reported in Table~\ref{tab:glue_ft}.

\begin{table}[ht]
\centering
\caption{Results on the GLUE benchmark with full fine-tuning of \texttt{Roberta-Large} (185M). Only the self-attention matrices are updated, while all other parameters remain frozen. The best result per column is shown in bold and the second-best is underlined.}
\label{tab:glue_ft}
\renewcommand{\arraystretch}{1.1}
\resizebox{\linewidth}{!}{
\begin{tabular}{l|ccccccc|c}
\toprule
\multirow{2}*{\bf Method} & {\bf MNLI} & {\bf SST-2} & {\bf CoLA} & {\bf QQP } & {\bf QNLI} & {\bf RTE}  & {\bf MRPC}  & {\bf ALL} \\
~  & {Acc} & {Acc} & {Mcc} & {Acc} & {Acc} & {Acc} & {Acc} & {Avg } \\
\midrule
\texttt{AdamW} &
\textbf{90.13} & 95.87 & \underline{68.66} & \textbf{92.24} & \textbf{94.00} & \underline{83.03} & \underline{89.95} & \underline{87.70} \\
\texttt{Muon} &
89.94 & \underline{95.18} & \textbf{68.73} & 91.88 & 93.32 & 80.51 & \underline{89.95} & 87.07 \\
\texttt{SOAP} &
\underline{90.03} & 94.84 & 67.53 & 91.86 & 93.30 & 81.23 & 88.73 & 86.79 \\
\texttt{DyKAF} &
90.01 & \textbf{96.45} & 68.44 & \underline{92.16} & \underline{93.54} & \textbf{84.48} & \textbf{90.93} & \textbf{88.00} \\
\bottomrule
\end{tabular}
}
\end{table}

\paragraph{LLM downstream adaptation.} 
To test transferability in large-scale settings, we further evaluate \texttt{Qwen3-8B} model \citep{yang2025qwen3} across classical LLM benchmarks: MathQA \citep{yu2023metamath}, GSM8K \citep{cobbe2021training}, HellaSwag \citep{zellers2019hellaswag}, BoolQ \citep{clark2019boolq} and ARC-Challenge \citep{clark2018think}. Unlike the GLUE case, here we employ LoRA adapters \citep{hu2021lora} for practical efficiency, reflecting a more realistic adaptation pipeline for modern LLMs. This benchmark emphasizes sample efficiency and robustness to heterogeneous reasoning and comprehension tasks. The corresponding results are presented in Table~\ref{tab:qwen_ft}.

\begin{table}[h!]
\centering
\caption{Results on LLM fine-tuning with \texttt{Qwen3-8B} using LoRA on self-attention matrices. We evaluate different optimizers across classical LLM benchmarks. The best result per column is shown in bold and the second-best is underlined.}
\label{tab:qwen_ft}
\renewcommand{\arraystretch}{1.1}
\resizebox{\linewidth}{!}{
\begin{tabular}{l|ccccc|c}
\toprule
\multirow{2}*{\bf Method} 
& {\bf MathQA} & {\bf GSM8K} & {\bf HellaSwag} & {\bf BoolQ} & {\bf ARC} & {\bf ALL} \\
~ & {Acc} & {Acc} & {Acc} & {Acc} & {Acc} & {Avg} \\
\midrule
\texttt{AdamW} &
\underline{43} & \underline{40} & \underline{69} & 63 & 85 & 60.0 \\
\texttt{Muon} &
\underline{43} & \textbf{42} & \underline{69} & \underline{64} & 85 & \underline{60.6} \\
\texttt{SOAP} &
41 & \textbf{42} & 68 & 63 & \textbf{89} & \underline{60.6} \\
\texttt{DyKAF} &
\textbf{47} & 39 & \textbf{70} & \textbf{67} & \underline{86} & \textbf{61.8} \\
\bottomrule
\end{tabular}
}
\end{table}

\paragraph{Fine-tuning experiment results.} 
Across both setups DyKAF achieves the strongest overall performance. On the GLUE benchmark with full fine-tuning of \texttt{Roberta-Large}, DyKAF attains the highest average score, consistently matching or surpassing baselines on most tasks. On the large-scale LLM experiments with \texttt{Qwen3-8B}, DyKAF again achieves the best average result, with particularly strong gains on reasoning-oriented datasets such as MathQA and BoolQ. These findings demonstrate that DyKAF provides competitive advantages in both medium-scale and large-scale fine-tuning regimes, outperforming widely used optimizers including AdamW, Muon, and SOAP.

\subsection{Pretraining}
\label{sub:pt}
We evaluate DyKAF (Algorithm~\ref{alg:dykaf}) in two pretraining scenarios, spanning small- to medium-scale (NanoGPT-like) and larger-scale (LLaMA-124M on FineWeb).
In all fine-tuning experiments with  DyKAF optimizer we set \texttt{rank1\_second\_moment = False} in Algorithm \ref{alg:dykaf}, supported by an ablation study (Appendix~\ref{app:ablations_llama}), which shows that, unlike in fine-tuning, pretraining benefits from retaining a full-rank second-moment estimate. Other hyperparameter details are provided in Appendix~\ref{app:pt}.

\paragraph{Character-level language modeling.} 
We first evaluate small-scale \texttt{GPT-base} model using Shakespeare-char dataset ($\sim 160$M tokens) \citep{caldas2018leaf} following the NanoGPT setup \citep{Karpathy2022}. We consider models of varying sizes ($0.41$M, $1.55$M, $3.24$M parameters) and report minimum validation loss for each scale. Results are presented in Table~\ref{tab:pretrain_gptbase}.

\begin{table}[h!]
\centering
\caption{Pretraining on Shakespeare-char with \texttt{GPT-base} models of varying sizes. We report minimum validation loss for different model sizes. The best result per column is shown in bold and the second-best is underlined.}
\label{tab:pretrain_gptbase}
\renewcommand{\arraystretch}{1.1}
\resizebox{\linewidth}{!}{
\begin{tabular}{l|ccc|c}
\toprule
\multirow{2}*{\bf Method} & {\bf 0.41M} & {\bf 1.55M} & {\bf 3.24M} & {\bf ALL} \\
~ & {min val loss} & {min val loss} & {min val loss} & {Avg} \\
\midrule
\texttt{AdamW} & \underline{1.5923} & 1.5893 & 1.5614 & 1.5810 \\
\texttt{Muon} & 1.6430 & 1.6218 & 1.6160 & 1.6269 \\
\texttt{SOAP} & 1.5954 & \underline{1.5749} & \underline{1.5474} & \underline{1.5726} \\
\texttt{DyKAF} & \textbf{1.5797} & \textbf{1.5671} & \textbf{1.5346} & \textbf{1.5605} \\
\bottomrule
\end{tabular}
}
\end{table}

\paragraph{Large-scale web corpus pretraining.}
We next follow the large-scale setup introduced in \citep{semenov2025benchmarking}, pretraining \texttt{LLaMA-124M} \citep{dubey2024llama} on the FineWeb dataset ($\sim$100B tokens) \citep{penedo2024fineweb}. In all previous experiments, competing optimizers were re-implemented and tuned under identical conditions. For this large-scale setting, however, we directly adopted the reported results of AdamW, Muon, and SOAP from \citep{semenov2025benchmarking}. DyKAF was run with the same hyperparameters as SOAP in that benchmark, since pretraining runs at this scale are prohibitively expensive, making a dedicated hyperparameter sweep infeasible. Results are deferred to Table~\ref{tab:pretrain_llama}.

\begin{table}[h!]
\centering
\caption{Pretraining on FineWeb with \texttt{LLaMA-124M}. Reported values are validation losses after different number of processed training tokens. The best result per column is shown in bold and the second-best is underlined.}
\label{tab:pretrain_llama}
\renewcommand{\arraystretch}{1.1}
\resizebox{\linewidth}{!}{
\begin{tabular}{l|cccccc|c}
\toprule
\multirow{2}*{\bf Method} & {\bf 1B} & {\bf 2.1B} & {\bf 4.2B} & {\bf 6.3B } & {\bf 8.4B} & {\bf 16.8B} & {\bf ALL} \\
~ & {val loss} & {val loss} & {val loss} & {val loss} & {val loss} & {val loss} & {Avg} \\
\midrule
\texttt{AdamW} &
3.5098 & 3.3992 & 3.3128 & 3.2752 & 3.2522 & \underline{3.2083} & 3.3263 \\
\texttt{Muon} &
\underline{3.4626} & \underline{3.3662} & \underline{3.2953} & \underline{3.2639} & 3.2445 & 3.2111 & \underline{3.3073} \\
\texttt{SOAP} &
3.4731 & 3.3735 & 3.3006 & \underline{3.2637} & \underline{3.2437} & 3.2149 & 3.3116 \\
\texttt{DyKAF} &
\textbf{3.4553} & \textbf{3.3568} & \textbf{3.2832} & \textbf{3.2464} & \textbf{3.2319} & \textbf{3.1986} & \textbf{3.2954}\\
\bottomrule
\end{tabular}
}
\end{table}
\paragraph{Pretraining experiment results.} 
Pretraining experiments require substantially larger computational resources than fine-tuning, with models trained on billions of tokens and long wall-clock times. Across all settings DyKAF achieves the strongest results, outperforming baseline optimizers. On the small-scale \texttt{GPT-base} pretraining, DyKAF consistently obtains the lowest validation loss across model sizes. In the large-scale \texttt{LLaMA-124M} pretraining, DyKAF again surpasses the baselines. Importantly, in this most resource-intensive experiment we reuse the SOAP hyperparameters without any additional tuning, nevertheless DyKAF remains robust and yields state-of-the-art results, highlighting the stability of the method under challenging large-scale training conditions.

\section{Related Work}
\label{sec:related_work}

K-FAC approximates the layerwise Fisher with a Kronecker product, lowering cost while preserving key curvature \citep{martens2015optimizing}. Extensions to convolutional layers exploit spatial sharing for practical speedups \citep{grosse2016kronecker,tornstad2020evaluating}. Empirically, K-FAC–style preconditioners offer a strong accuracy–efficiency trade-off, motivating low-rank and block variants \citep{ren2021kronecker,koroko2023efficient}.

Shampoo builds on K-FAC by preconditioning with matrix powers, using the -1/4 exponent to shape gradient updates and improve convergence stability \citep{gupta2018shampoo}. A distributed, numerically robust implementation \citep{shi2023distributed} achieves state-of-the-art results and wins AlgoPerf \citep{Dahl2023AlgoPerf,Kasimbeg2025AlgoPerfResults}. Recent theory shows that the square of Shampoo’s approximation equals one step of a power method toward the optimal Kronecker-factored Fisher, clarifying its performance \citep{morwani2024new}.

SOAP builds on Shampoo by applying an AdamW-like step in a Shampoo-preconditioned parameter space and replacing costly matrix roots with fast QR decompositions \citep{vyas2024soap}. In large-scale experiments, SOAP shows faster convergence than AdamW and Shampoo across standard benchmarks, supporting its use as a strong default in practice \citep{semenov2025benchmarking,wen2025fantastic}.

Recent work proposes targeted modifications to the Kronecker-based preconditioners used in SOAP and Shampoo, aiming to stabilize curvature and keep costs modest. Fisher-Orthogonal Projection adds structured projections to stabilize Fisher estimates in large models \citep{lu2025fisher}. Two algorithms refine the Fisher approximation with block structure and low rank while maintaining Kronecker-like update cost \citep{gong2025towards}. KL-aligned preconditioners further improve the performance of second-order steps \citep{lin2025understanding}. Overall, these studies sharpen Kronecker-based approximations with simple mechanisms and practical overheads.

\section{Conclusion}
This paper proposes a dynamical Kronecker approximation of the empirical Fisher matrix within a SOAP-style preconditioner, maintained via a projector-splitting algorithm, enabling stable curvature-aware training with minimal overhead while preserving layerwise structure. Experiments on language understanding, LLM adaptation, and pretraining show improved performance over strong baselines without extensive hyperparameter sweeps.

\bibliography{refs}

\clearpage
\appendix
\thispagestyle{empty}

\onecolumn
\aistatstitle{
Supplementary Materials}

\section{Empirical Studies on Hessian Accuracy}
\label{app:hessian_calc}
\subsection{Hessian of the Softmax Cross-Entropy Loss}
\label{app_sub:hess_calc}

\paragraph{Setup.}
For a single training pair $(x,y)$ with $x\in\mathbb{R}^{n}$ and one-hot  
$y\in\{e_1,\dots,e_m\}$, let
\[
z = W x, \qquad
p = \operatorname{softmax}(z), \qquad
p_k = \frac{\exp(z_k)}{\sum_{j=1}^{m}\exp(z_j)},
\]
where $W\in\mathbb{R}^{m\times n}$ and $z,p\in\mathbb{R}^{m}$.
The per-sample cross-entropy loss is
\[
L(W) \;=\; -\sum_{k=1}^{m} y_k \,\log p_k.
\]

\paragraph{Gradient.}
The softmax Jacobian is
\[
\frac{\partial p_i}{\partial z_j} = p_i(\delta_{ij}-p_j).
\]
Using it together with $\partial(-\log p_k)/\partial p_k=-1/p_k$,
\[
\frac{\partial L}{\partial z_j}
    = \sum_{k=1}^{m} \Bigl(-\frac{y_k}{p_k}\Bigr)
      p_k(\delta_{kj}-p_j)
    = p_j - y_j
    \;\;\Longrightarrow\;\;
    \nabla_{z} L = p - y.
\]
Because $z = W x$ and $\partial z_j/\partial W_{jk}=x_k$,
\[
\nabla_{W} L
    = (p-y)\,x^{\!\top}\in\mathbb{R}^{m\times n}.
\]

\paragraph{Hessian.}
Start from the element-wise gradient  
\[
\frac{\partial L}{\partial W_{ij}}
    =(p_{i}-y_{i})\,x_{j},
\qquad i=1,\dots,m,\; j=1,\dots,n,
\]
and differentiate once more with respect to another entry \(W_{kl}\):
\[
\frac{\partial^{2}L}{\partial W_{ij}\,\partial W_{kl}}
    =x_{j}\;
     \frac{\partial(p_{i}-y_{i})}{\partial W_{kl}}
    =x_{j}\;
     \frac{\partial p_{i}}{\partial W_{kl}}.
\]

Because \(z = W x\), the logit derivatives are
\[
\frac{\partial z_{r}}{\partial W_{kl}}
    =x_{l}\,\delta_{rk}.
\]
Using the soft-max Jacobian
\(\frac{\partial p_{i}}{\partial z_{r}}
      =p_{i}(\delta_{ir}-p_{r})\),
\[
\frac{\partial p_{i}}{\partial W_{kl}}
    =\sum_{r=1}^{m}
      \frac{\partial p_{i}}{\partial z_{r}}\,
      \frac{\partial z_{r}}{\partial W_{kl}}
    =p_{i}\bigl(\delta_{ik}-p_{k}\bigr)\,x_{l}.
\]

Substituting,
\[
\frac{\partial^{2}L}{\partial W_{ij}\,\partial W_{kl}}
      =p_{i}\bigl(\delta_{ik}-p_{k}\bigr)\;
       x_{j}\,x_{l}
\qquad
(i,j,k,l).
\]
Let
\(H^{(z)}=\operatorname{diag}(p)-p\,p^{\!\top}\in\mathbb{R}^{m\times m}\)
and \(S = x\,x^{\!\top}\in\mathbb{R}^{n\times n}\).
Vectorising \(W\) into \(w=\operatorname{vec}(W)\in\mathbb{R}^{mn}\),
the Hessian becomes
\[
\nabla^{2}_{W}L
    =H^{(z)}\otimes S
    \;\in\;\mathbb{R}^{(mn)\times(mn)},
\]
i.e.\ a Kronecker product of the \emph{predictive covariance}
\(H^{(z)}\) and the \emph{input covariance} \(S\).

\paragraph{Dataset loss.}
For a dataset $\{(x_i,y_i)\}_{i=1}^{N}$ define the empirical loss
\[
L(W) = -\frac{1}{N}\sum_{i=1}^{N}\sum_{k=1}^{m} y_{ik}\,\log p_{ik},
\quad p_i = \operatorname{softmax}(W x_i).
\]
Its first and second derivatives are the averages
\[
\nabla_{W} L
    = \frac{1}{N}\sum_{i=1}^{N} (p_i - y_i)\,x_i^{\!\top},
\qquad
\nabla^{2}_{W} L
    = \frac{1}{N}\sum_{i=1}^{N}
      \Bigl(\operatorname{diag}(p_i)-p_i p_i^{\!\top}\Bigr)
      \otimes
      \bigl(x_i x_i^{\!\top}\bigr).
\]
Hence the full-batch Hessian remains a sum of Kronecker products of
input covariances $x_i x_i^{\!\top}$ and predictive covariances
$\operatorname{diag}(p_i)-p_i p_i^{\!\top}$.

\subsection{Experiment Details}
We examine the accuracy of Hessian approximation on a controlled setup: multiclass classification of the \texttt{mushrooms} dataset \citep{chang2011libsvm} with an MLP and softmax output.
In this case the exact Hessian can be computed analytically (see Appendix~\ref{app_sub:hess_calc}), enabling a direct comparison of the Fisher matrix estimates from SOAP and DyKAF with real Hessian matrix.
We evaluate the difference for varying numbers of available training samples, with optimizer hyperparameters tuned via Optuna package \citep{akiba2019optuna}. 
For both SOAP and DyKAF, we build the full Fisher matrix approximation using the following formula:
\begin{equation*}
	\widetilde{F}_{\mathrm{method}} = (Q_L \otimes Q_R)\mathrm{\diag}(\mathrm{vec}(V_t))(Q_L \otimes Q_R)^{\top},
\end{equation*}
where $Q_L, Q_R$ and $V_t$ are taken from \Algref{alg:dykaf}.
Motivated by \citep{martens2015optimizing}, we use $\widetilde{F}_{method}$ as Hessian approximation and report the Frobenius norm of the difference between the real Hessian and its estimates $\|H_{real}-\tilde F_{\text{method}}\|$. 
Figure~\ref{fig:hess_diff} shows that DyKAF consistently yields a more accurate Hessian approximation than SOAP across various settings.

\section{Proof of Proposition \ref{thm:thm_dynamical}}
\label{apx:thm_dynamical}

\begin{proof}
Let for the sake of brevity $\tilde A=\mathcal{R}(A)$ in this section. 
Applying rearrangement operator to $F_0$ and $F_1$, we obtain
\[\tilde F_0=\mathcal{R}(F_0)=\mathcal{R}(L_{0}\otimes R_0+E_0)=s_0\ell_0r_0^{\top}+\tilde E_0, \]
\[\tilde F_1 =\mathcal{R}(F_1)=\mathcal{R}(F_0+\Delta F)=\mathcal{R}(A_1\otimes B_1+E_1)=sab^{\top}+\tilde E_1,\] 
where $\Delta F=F_1-F_0$; the vectors $\ell_0 = \mathrm{vec}(L_0)/\|L_0\|, r_0=\mathrm{vec}(R_0)/\|R_0\|, a=\mathrm{vec}(A_1)/\|A_1\|, b=\mathrm{vec}(B_1)/\|B_1\|$ have norm 1; the values $s_0=\|A_0\otimes B_0\|=\|L_0\otimes R_0\|, s=\|A_1\otimes B_1\|$ are positive scalars. Let us denote $\tilde E=\tilde E_1-\tilde E_0$.

We want to apply  projector-splitting to approximate $\mathcal{R}(F_1)=\mathcal{R}(F_0)+\mathcal{R}(\Delta F)$, where the current rank-1 approximation to $\mathcal{R}(F_0)$ is $s_0 \ell_0 r_0^T$. The target dynamical low-rank approximation to $\tilde F_1$ is $\ell_1 s_1 r_1^T$, where $s_1=\ell_1^T(s_0\ell_0r_0^T+\mathcal{R}(\Delta F))r_1=\ell_1^T(\tilde F_1-\tilde E_0)r_1$ by formula from Algorithm \ref{alg:proj_split}. Now we aim to sequentially obtain the necessary formula, separately deriving all the projector-splitting components.

Using Algorithm \ref{alg:proj_split}, we obtain:
\[\ell_1=\frac{\ell_0s_0+\mathcal{R}(\Delta F)r_0}{\|\ell_0s_0+\mathcal{R}(\Delta F)r_0\|}=\frac{(\ell_0r_0^{\top}s_0+\mathcal{R}(\Delta F))r_0}{\|(\ell_0r_0^{\top}s_0+\mathcal{R}(\Delta F))r_0\|}=\frac{(\tilde F_1-\tilde E_0)r_0}{\|(\tilde F_1-\tilde E_0)r_0\|}=\frac{(sab^\top+\tilde E)r_0}{\|(sab^\top+\tilde E)r_0\|}.\]

Let us denote $x=r_0^\top b$ and $y=r_0^\top\tilde E^\top a$.
\[d=\|(sab^\top+\tilde E)r_0\|^2
=r_0^\top(sba^\top+\tilde E^\top)(sab^\top+\tilde E)r_0
=s^2(r_0^\top b)^2+2s(r_0^\top b)(r_0^\top\tilde E^\top a)+O(\varepsilon^2)=
s^2x^2+2sxy+O(\varepsilon)^2.\]
\[d^{-1}=\frac{1}{s^2x^2}(1-2\frac{y}{sx}+O(\varepsilon^2))=\frac{1}{s^2x^2}-2\frac{y}{s^3x^3}+O(\varepsilon^2).\]

Let us substitute the this expression into the formula for $\ell_1\ell_1^{\top}(\tilde F_1-\tilde E_0)$:
\[\begin{aligned}&\ell_1\ell_1^{\top}(\tilde F_1-\tilde E_0)=\frac{(ab^{\top}s+\tilde E)r_0r_0^{\top}(ba^{\top}s+\tilde E^{\top})(ab^{\top}s+\tilde E)}{\|(ab^{\top}s+\tilde E)r_0\|^2}=
\\&=\left(ab^\top s^3(b^\top r_0)^2+s^2(r_0^\top b)(\tilde E r_0b^\top)+s^2(b^\top r_0)(ar_0^\top\tilde E^\top ab^\top)+s^2(r_0b^\top)^2(aa^\top \tilde E)+O(\varepsilon^2)\right)d^{-1}=
\\& =\left(ab^\top s^3x^2+s^2x(\tilde E r_0b^\top)+s^2x(ar_0^\top\tilde E^\top ab^\top)+s^2x^2(aa^\top \tilde E)+O(\varepsilon^2)\right)\left(\frac{1}{s^2x^2}-2\frac{(r_0^\top\tilde E^\top a)}{s^3x^3}+O(\varepsilon^2)\right)=
\\& =ab^\top s -\frac{2ab^Ty}{x}+\frac{\tilde Er_0b^\top}{x}+\frac{ab^\top y}{x}+aa^\top \tilde E +O(\varepsilon^2)\end{aligned}.\]

The same way, \[r_1=\frac{(sab^\top+\tilde E)^{\top}\ell_0}{\|(sab^\top+\tilde E)^{\top}\ell_0\|}.\] 
Let us denote $z=\ell_0^\top a$ and $p=b^\top \tilde E^\top \ell_0$.
\begin{equation*}
\|(sab^\top+\tilde E)^{\top}\ell_0\|^{-2}=\frac{1}{s^2z^2}-2\frac{p}{s^3z^3}+O(\varepsilon^2).
\end{equation*}
Now, we are able to rewrite $r_1r_1^{\top}$ and fully assemble $s_1l_1r_1^{\top}$.
\begin{equation*}
\begin{gathered}
r_1r_1^\top=\left((sba^\top+\tilde E^{\top})\ell_0\ell_0^\top(sab^\top+\tilde E)\right)\left(\frac{1}{s^2z^2}-2\frac{p}{s^3z^3}+O(\varepsilon^2)\right)=
\\=\left(s^2z^2bb^\top+szb\ell_0^\top\tilde E+sz\tilde E^\top\ell_0 b^\top+O(\varepsilon^2)\right)\left(\frac{1}{s^2z^2}-2\frac{p}{s^3z^3}+O(\varepsilon^2)\right)=
\\=bb^\top-2\frac{pbb^T}{sz}+\frac{b\ell_0^\top\tilde E}{sz}+\frac{E^\top\ell_0 b^\top}{sz}+O(\varepsilon^2)
\end{gathered}
\end{equation*}

Finally, we estimate $\ell_1 s_1 r_1^T$:
\[\begin{aligned}&\ell_1s_1r_1^T=\ell_1\ell_1^T(\tilde F_1-\tilde E_0)r_1r_1^\top=\\&=\left(ab^\top s -\frac{2ab^Ty}{x}+\frac{\tilde Er_0b^\top}{x}+\frac{ab^\top y}{x}+aa^\top \tilde E +O(\varepsilon^2)\right)\left(bb^\top-2\frac{pbb^T}{sz}+\frac{b\ell_0^\top\tilde E}{sz}+\frac{E^\top\ell_0 b^\top}{sz}+O(\varepsilon^2)\right)=
\\&=ab^\top s-\frac{2pab^\top}{z}+\frac{a\ell_0^\top\tilde E}{z}+\frac{pab^\top}{z}-\frac{2yab^\top}{x} +\frac{\tilde Er_0b^\top}{x}+\frac{yab^\top}{x}+aa^\top \tilde Ebb^\top +O(\varepsilon^2)=
\\&=ab^\top s+\frac{1}{\ell_0^\top a}\left(-(b^\top\tilde E^\top \ell_0)ab^\top+a\ell_0^\top\tilde E\right)+\frac{1}{r_0^\top b}\left(-(r_0^\top \tilde E^\top  a)ab^\top +\tilde Er_0b^\top\right)+aa^\top \tilde Ebb^\top +O(\varepsilon^2)=
\\& = ab^\top + \widehat E_1 +O(\varepsilon^2).\end{aligned}\]

Let us estimate the norms of the summands constituting $\tilde E_1$, which contain the first-order error $\tilde E$.
\[\|aa^\top \tilde Ebb^\top\|\leq \|aa^\top \|\|\tilde E\|\|bb^\top\|=\|\tilde E\|.\]
For any matrices $C$ and $D$ of correct sizes, $\|CD\|_F\leq \|C\|_F\|D\|_2$, where $\|\cdot\|_F$ is Frobenius norm (which we have always denoted as $\|\cdot \|$) and $\|\cdot\|_2$ is the spectral norm. The spectral norm of orthoprojector equals 1: $\left\|I-bb^\top\right\|_2=1$. Therefore, 
\begin{equation*}
\left\|\frac{1}{\ell_0^\top a}\left(-(b^\top\tilde E^\top \ell_0)ab^\top+a\ell_0^\top\tilde E\right)\right\|_F = \frac{\left\|a\ell_0^\top\tilde E\left(I-bb^\top\right)\right\|_F}{|\ell_0^Ta|} \leq \frac{\left\|a\ell_0^\top\right\|_F\left\|\tilde E\right\|_F\left\|I-bb^\top\right\|_2}{|\ell_0^Ta|} = \frac{\left\|\tilde E\right\|_F}{|\ell_0^Ta|}.
\end{equation*}

Analogously, 
\begin{equation*}
\left\|\frac{1}{r_0^\top b}\left(-(r_0^\top \tilde E^\top  a)ab^\top +\tilde Er_0b^\top\right)\right\|\leq \frac{\left\|\tilde E\right\|}{|r_0^Tb|}.
\end{equation*}

Therefore,
\[\begin{aligned}&\ell_1s_1r_1^T= ab^\top s +  \widehat E_{1} +O(\varepsilon^2),\end{aligned}\]
where $\|\widehat E_{1}\|\leq \left(1+\frac{1}{|\ell_0^Ta|}+\frac{1}{|r_0^Tb|}\right)\|\tilde E\|$.

Let us remember that 
\begin{equation*}
    |\ell_0^Ta|=\frac{|\mathrm{tr}(A_0^TA_1)|}{\|A_0\|\|A_1\|}
\end{equation*} and 
\begin{equation*}
    |r_0^Tb|=\frac{|\mathrm{tr}(B_0^TB_1)|}{\|B_0\|\|B_1\|}.
\end{equation*}
We know from Cauchy-Schwarz inequality that
\begin{equation*}
    \frac{|\mathrm{tr}(A_0^TA_1)|}{\|A_0\|\|A_1\|}\leq 1.
\end{equation*} 
Using the condition from the proposition, we get 
\begin{equation*}
\frac{|\mathrm{tr}((A_0\otimes B_0)^\top (A_1\otimes B_1))|}{\|A_0\otimes B_0\|\|A_1\otimes B_1\|}\geq c>0
\quad \Rightarrow \quad \frac{|\mathrm{tr}(B_0^TB_1)|}{\|B_0\|\|B_1\|} \geq \frac{|\mathrm{tr}(A_0^TA_1)||\mathrm{tr}(B_0^TB_1)|}{\|A_0\|\|A_1\|\|B_0\|\|B_1\|} \geq c.
\end{equation*}
Analogously, 
\begin{equation*}
    |\ell_0^Ta|=\frac{|\mathrm{tr}(A_0^TA_1)|}{\|A_0\|\|A_1\|}\geq c.
\end{equation*} 
Therefore, 
\begin{equation*}\begin{aligned}&\|\widehat E_{1}\|\leq 
\left(1+\frac{1}{|\ell_0^Ta|}+\frac{1}{|r_0^Tb|}\right)\|\tilde E\|= 
\left(1+\frac{1}{|\ell_0^Ta|}+\frac{1}{|r_0^Tb|}\right)\|E\|\leq
\left(1+\frac{2}{c}\right) \| E\|\leq \left(1+\frac{2}{c}\right) (\| E_1\|+\|E_0\|). \end{aligned}\end{equation*}
\end{proof}

\section{Proof of Proposition \ref{thm:thm_coherence}}
\label{apx:thm_coherence}
\begin{proof}
Let $g_i=\mathrm{vec}(G_i)$. 

Then it holds that
\[\left\|\left(\sum_i^tG_iG_i^{\top}\right)^{1/2}\right\|^2=\mathrm{tr}\left(\sum_i^tG_iG_i^{\top}\right)=\sum_i^t\|G_i\|^2.\]
In the same way,
\[\left\|\left(\sum_i^tG_i^{\top}G_i\right)^{1/2}\right\|^2=\sum_i^t\|G_i\|^2.\]
Then we have
\[ \left\|\left(\sum_i^tG_iG_i^{\top}\right)^{1/2}\otimes\left(\sum_i^tG_i^{\top}G_i\right)^{1/2}\right\|^2=\left\|\left(\sum_i^tG_iG_i^{\top}\right)^{1/2}\right\| ^2\left\|\left(\sum_i^tG_i^{\top}G_i\right)^{1/2}\right\|^2=\left(\sum_i^t\|G_i\|^2\right)^2.\]

Let $\mu$ denote the mutual the coherence of gradients $\mu=\max_{i\not=j}|g_i^{\top}g_j|/{\|g_i\|\|g_j\|}$. Then
\[\left\|\sum_i^tg_ig_i^{\top}\right\|^2=\sum_{i}^t\sum_{j}^t(g_i^{\top}g_j)^2=\sum_i^t\|g_i\|^4+\sum_{i}^t\sum_{j\not=i}^{t}(g_i^{\top}g_j)^2\leq \sum_i^t \|G_i\|^4+\mu^2 \sum_{i}^t\sum_{j\not=i}^t\|G_i\|^2\|G_j\|^2.\]
Finally, 
\[\left\|\left(\sum_i^tG_iG_i^{\top}\right)^{1/2}\otimes\left(\sum_i^tG_i^{\top}G_i\right)^{1/2}\right\|^2-\left\|\sum_i^tg_ig_i^{\top}\right\|^2\geq (1-\mu^2) \sum_{i}^t\sum_{j\not=i}^t\|G_i\|^2\|G_j\|^2.\]
\end{proof}
\section{Proof of Proposition \ref{lemma:init}} \label{apx:proof_init}
\begin{proof}
	Let $G=U\Sigma V^{\top}$ be SVD of $G$. Applying rearrangement, we have
	\[\|\mathrm{vec}(G)\mathrm{vec}(G)^{\top} - L\otimes R\|=\|G\otimes G - \mathrm{vec}(L) \mathrm{vec}(R)^{\top}\|=\|(U\otimes U)(\Sigma \otimes \Sigma) (V\otimes V)^{\top}- \mathrm{vec}(L) \mathrm{vec}(R)^{\top}\|\]
	The largest singular value of $G\otimes G$ is $\sigma_1^2$, the corresponding eigenvectors are $u_1 \otimes u_1$ and $v_1 \otimes v_1$. Then the best rank-1 approximation of $G\otimes G$ is
    \begin{equation*}
        \mathrm{vec}(L)\mathrm{vec}(R)^{\top}= \sigma_1^2 (u_1 \otimes u_1)(v_1 \otimes v_1)^{\top}.
    \end{equation*}
    If we reshape, $L\otimes R = \sigma_1^2 (u_1u_1^{\top}) \otimes (v_1v_1^{\top})$. So $L=\sigma_1u_1u_1^{\top}, R=\sigma_1v_1v_1^{\top}$ is the solution of the minimization problem.
\end{proof}
\section{Proof of Proposition \ref{thm:fisher_diag}}
\label{app:zlp}
Let $\tilde{F} = (Q_A \otimes Q_B)^{\top}F(Q_A \otimes Q_B)$.
	Similarly, define $\tilde{E} = (Q_A \otimes Q_B)^{\top}E(Q_A \otimes Q_B)$. Because of the unitary invariance of the Frobenius norm we have
	\begin{equation*}
		\left\|{\mathrm{diag}(F)}\right\|_{{F}}^2 + \left\|{\mathrm{off}(F)}\right\|_{{F}}^2 = \left\|{\mathrm{diag}(\tilde{F})}\right\|_{{F}}^2 + \left\|{\mathrm{off}(\tilde{F})}\right\|_{{F}}^2.
	\end{equation*}
    Swapping parts of the equation, we finally obtain
	
	\begin{equation*}
		\left\|{\mathrm{diag}(\tilde{F})}\right\|_{{F}}^2 - \left\|{\mathrm{diag}(F)}\right\|_{{F}}^2 = \left\|{\mathrm{off}(F)}\right\|_{{F}}^2 - \left\|{\mathrm{off}(\tilde{F})}\right\|_{{F}}^2 \geqslant 0.
	\end{equation*}
\newpage
\section{Ablation Studies}
\label{sec:ablation}
\subsection{Performance evaluation}
\label{sec:perf_eval}
In this section we report the asymptotics of projector-splitting algorithms we use and describe the theoretical overhead of out method.
Our approach to dynamically update $L_t, R_t$ by given $G_t$, described in \Algref{alg:kron_proj_split} requires 6 matrix-matrix multiplications, 5 scalar product operations (e.g matrix normalization or Frobenius scalar product of two matrices) in comparison with 2 matrix multiplications in the case of Shampoo optimizer. 
At the first glance it may seem that we need more normalizations in \Algref{alg:kron_proj_split}, but we can omit normalization of $L_t$ and $R_t$ at first two steps and make normalization only at steps 3 and 4 of the algorithm. 
In the case when we aim to preserve rank-1 approximation of the second Adam moment (\texttt{rank1\_second\_moment = True}), our computational overhead equals to 4 matrix-vector multiplications and 5 vector operations (e.g normalization or scalar product).
In practice such an increase in the number of operations does not play a vital role in the overall wall time, but increases it up to 20\%.
Additionally, we measure the wall time of the fine-tuning of Qwen3-8B on LLM fine-tuning benchmarks (see Table \ref{tab:ft_runtime}).

\begin{table}[h!]
\centering
\caption{Wall-clock runtime (minutes) for Qwen3-8B fine-tuning across standard LLM benchmarks. All runs use identical hardware, dataloader, and evaluation protocols.}
\label{tab:ft_runtime}
\renewcommand{\arraystretch}{1.1}
\begin{tabular}{l|ccccc|c}
\toprule
\bf Method & \bf MathQA & \bf GSM8K & \bf HellaSwag & \bf BoolQ & \bf ARC & \bf AVG \\
\midrule
SOAP & 83 & 84 & 87 & 85 & 86 & 85 \\
DyKAF & 106 & 107 & 109 & 106 & 107 & 107 \\
\bottomrule
\end{tabular}
\end{table}

\subsection{Ablation on \texttt{rank1\_second\_moment}}
\label{app:ablations_llama}

We ablate the effect of enabling the \texttt{rank1\_second\_moment} flag in Algorithm \ref{alg:dykaf}, corresponding to a low-rank approximation $V \approx V_L V_R^\top$ of the Adam second-moment matrix $V$. This approximation reduces memory overhead and, as shown below, may also yield quality improvements depending on the regime. Hyperparameters follow the same tuning protocol as in the main experiments (see Appendix~\ref{app:ft} and Appendix~\ref{app:pt}).  

\paragraph{Fine-tuning.} 
Table~\ref{tab:ablation_ft_rank1} reports results on LLM downstream datasets. The rank-1 approximation generally performs on par with the full-rank update and in some cases yields small improvements, while in others it is slightly worse. This confirms that in fine-tuning regimes—where gradients are often low-rank—the approximation provides memory savings and can be competitive in quality, similar to how parameter-efficient methods such as LoRA may outperform or match full fine-tuning \citep{veprikov2025weightlora}.  

\begin{table}[h!]
\centering
\caption{Ablation study of the rank-1 second moment on LLM fine-tuning benchmarks (Qwen3-8B). Best values are bold.}
\label{tab:ablation_ft_rank1}
\renewcommand{\arraystretch}{1.1}
\begin{tabular}{l|ccccc|c}
\toprule
\bf Method & \bf MathQA & \bf GSM8K & \bf HellaSwag & \bf BoolQ & \bf ARC & \bf ALL \\
\midrule
DyKAF (\texttt{rank1\_second\_moment=True})  & \textbf{47} & \textbf{39} & \textbf{70} & 67 & 85 & \textbf{61.6} \\
DyKAF (\texttt{rank1\_second\_moment=False}) & 45 & 35 & \textbf{70} & \textbf{70} & \textbf{87} & 61.4 \\
\bottomrule
\end{tabular}
\end{table}

\paragraph{Pretraining.} 
In contrast, Figure~\ref{fig:llama_ablation} shows that for LLaMA-124M pretraining on FineWeb, the rank-1 approximation increases validation loss across training, confirming that pretraining benefits from a full-rank second moment. 

\begin{figure}[h!]
    \centering
    \includegraphics[width=0.7\linewidth]{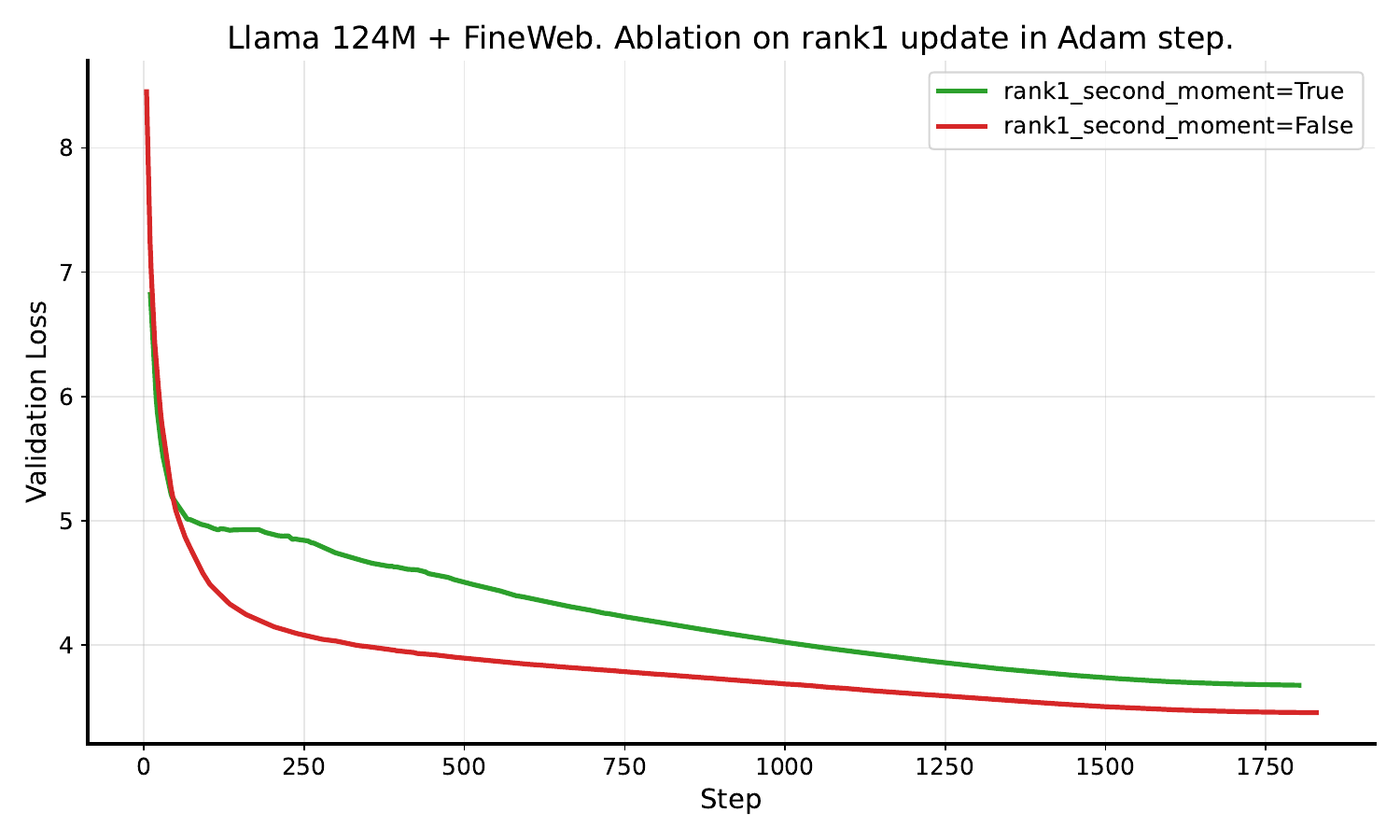}
    \caption{Ablation on \texttt{rank1\_second\_moment} during pretraining of LLaMA-124M on FineWeb. Unlike in fine-tuning, the rank-1 approximation consistently degrades performance, supporting the use of a full-rank second moment in pretraining.}
    \label{fig:llama_ablation}
\end{figure}

\paragraph{Summary.} 
Overall, the ablation indicates that the \texttt{rank1\_second\_moment} flag is beneficial in fine-tuning scenarios, where gradients tend to be low-rank and the approximation can reduce memory cost while sometimes improving generalization. In contrast, for large-scale pretraining, where gradient structure is richer, the full-rank second moment consistently provides better or more stable performance. We therefore recommend enabling the rank-1 approximation in fine-tuning tasks, but disabling it in pretraining.
\newpage
\section{Experimental Setups and Hyperparameters}
\label{app:exp}
\subsection{Fine-tuning Experiment (Section \ref{sub:ft})}
\label{app:ft}
\begin{table}[H]
\centering
\caption{GLUE full fine-tuning: training setup and hyperparameter sweep.}
\label{tab:glue_hparams}
\renewcommand{\arraystretch}{1.2}
\begin{tabular}{l l}
\toprule
\textbf{Parameter} & \textbf{Value} \\
\midrule
Hardware & RTX-2080 GPU 11GB \\
Backbone & \texttt{Roberta-Large} (185M) \\
Trainable params & Self-attention matrices (\texttt{query\_proj}, \texttt{key\_proj}, \texttt{value\_proj}) \\
Fine-tuning & Full FT (no adapters) \\
Batch size & 16 \\
Grad. accumulation & 2 \\
DType & bfloat16 \\
Num train epochs & $3$ for MNLI, SST-2, QQP, QNLI ; $10$ for CoLA, RTE, MRPC \\
Max seq length & $256$ \\
Weight decay & $1\!\times\!10^{-6}$ \\
Adam parameters & $\beta_1 = 0.9,, \beta_2 = 0.999, \varepsilon = 10^{-8}$ \\
LR scheduler & linear, warmup ratio $0.1$ \\
Sweep LRs & $\{10^{-5},\,10^{-4},\,5\!\times\!10^{-4},\,10^{-3},\,5\!\times\!10^{-3},\,10^{-2}\}$ \\
Update frequency (for SOAP and DyKAF) & $10$ \\
max evaluation samples &  $-1$ (all evaluation dataset was used) \\
\bottomrule
\end{tabular}
\end{table}

\begin{table}[H]
\centering
\caption{LLM fine-tuning with \texttt{Qwen3-8B}: training setup and hyperparameter sweep.}
\label{tab:qwen_hparams}
\renewcommand{\arraystretch}{1.2}
\begin{tabular}{l l}
\toprule
\textbf{Parameter} & \textbf{Value} \\
\midrule
Hardware & A100 GPU 40GB \\
Backbone & \texttt{Qwen3-8B} \\
Trainable params & Self-attention matrices (\texttt{q\_proj}, \texttt{k\_proj}, \texttt{v\_proj}, \texttt{o\_proj}) \\
Fine-tuning & LoRA (rank $r{=}8$, scaling $\alpha{=}32$, dropout $p=0.05$) \\
Batch size & $4$ \\
Grad. accumulation & $1$ \\
DType & bfloat16 \\
Max train steps & $3000$ \\
Max seq length & $1024$ \\
Weight decay & $1\!\times\!10^{-6}$ \\
Adam parameters & $\beta_1 = 0.9,, \beta_2 = 0.999, \varepsilon = 10^{-8}$ \\
LR scheduler & linear, warmup $100$ steps \\
Sweep LRs & $\{10^{-6}, 10^{-5},\,5\!\times\!10^{-5},\,10^{-4}, 5\!\times\!10^{-4}, 10^{-3}\}$\\
Update frequency (for SOAP and DyKAF) & $10$ \\
max evaluation samples &  $100$ \\
\bottomrule
\end{tabular}
\end{table}

\paragraph{Notes.} 
Batch size and maximum sequence length were chosen to fit the available hardware. The number of training epochs (GLUE) or maximum training steps (LLM), gradient accumulation, and the learning rate sweep range were selected as a trade-off between runtime efficiency and final model quality. All remaining parameters follow conventional settings commonly adopted in NLP fine-tuning, and were kept unchanged.
 
\subsection{Pre-train Experiment (Section \ref{sub:pt})}
\label{app:pt}

\begin{table}[H]
\centering
\caption{GPT-base pretraining (Shakespeare-char, $\sim$160M tokens): training setup and hyperparameter sweep.}
\label{tab:nanogpt_hparams}
\renewcommand{\arraystretch}{1.2}
\begin{tabular}{l l}
\toprule
\textbf{Parameter} & \textbf{Value} \\
\midrule
Hardware & RTX-2080 GPU 11GB \\
Backbone & \texttt{GPT-base} (character-level) \\
Dataset & Shakespeare-char ($\sim$160M tokens) \\
Vocabulary size & 96 \\
Batch size & $32$ \\
Grad. accumulation & $1$ \\
Sequence length & $256$ \\
Iterations & $30000$ \\
DType & bfloat16 \\
Dropout sweep & $\{0.05, 0.1, 0.15, 0.2\}$ \\
Weight decay & $0.1$ \\
Gradient clipping & $0.5$ \\
Adam parameters & $\beta_1 = 0.9, \beta_2 = 0.999, \varepsilon = 10^{-8}$ \\
LR scheduler & cosine, warmup $1000$ steps \\
Sweep LRs & $\{10^{-5}, 10^{-4}, 10^{-3}, 5\!\times\!10^{-3}, 10^{-2}\}$ \\
Update frequency (SOAP, DyKAF) & $10$ \\
\midrule
Model dimensions & \\
\ \, 0.41M params & n\_embd=$128$, n\_head=$4$, n\_layer=$2$ \\
\ \, 1.55M params & n\_embd=$256$, n\_head$=4$, n\_layer$=2$ \\
\ \, 3.24M params & n\_embd=$256$, n\_head=$4$, n\_layer=$4$ \\
\bottomrule
\end{tabular}
\end{table}
\begin{table}[H]
\centering
\caption{Large-scale LLaMA pretraining (FineWeb, $\sim$100B tokens): training setup.}
\label{tab:llama_hparams}
\renewcommand{\arraystretch}{1.2}
\begin{tabular}{l l}
\toprule
\textbf{Parameter} & \textbf{Value} \\
\midrule
Hardware & A100 GPU 80GB \\
Backbone & \texttt{LLaMA-124M} \\
Dataset & FineWeb ($\sim$100B tokens) \\
Batch size & $32$ \\
Gradient accumulation & $1$ \\
Sequence length & $512$ \\
Scheduler & Cosine, warmup $3000$ steps \\
DType & bfloat16 \\
Dropout & $0.0$ \\
Weight decay & $0.1$ \\
Grad. clipping & $0.5$ \\
Adam parameters & $\beta_1 = 0.9, \beta_2 = 0.999,  \varepsilon = 10^{-8}$ \\
LR & $10^{-3}$ \\
Update frequency (SOAP, DyKAF) & 10 \\
Model dimensions & n\_embd=$768$, n\_head=$12$, n\_layer$=12$ \\
\bottomrule
\end{tabular}
\end{table}
\paragraph{Notes.} 
Most hyperparameter choices in pretraining follow the same principles as in the fine-tuning experiments (Appendix~\ref{app:ft}): hardware-constrained batch size and context length, training steps set as a trade-off between runtime and quality, and otherwise conventional settings adopted in large-scale pretraining. For DyKAF we used the full-rank second moment (\texttt{rank1\_second\_moment = False}), consistent with our observation that pretraining benefits from richer curvature estimates. Importantly, in the large-scale LLaMA-124M FineWeb run (Table~\ref{tab:pretrain_llama}) we did not perform any hyperparameter tuning for DyKAF, instead reusing SOAP settings directly. This highlights the robustness of DyKAF to hyperparameter choices even under resource-intensive training regimes.
\newpage
\section{Higher Order Projector-splitting \label{apx:tensors}}
We can generalize Kronecker Fisher approximation to tensor-shaped parameters. Let the gradients $G_i$ be tensors of shape $n_1\times n_2 \times \dots \times n_d$. The Fisher matrix is defined the same way as in matrix case:
\[F_t = \sum_i vec(G_i)vec(G_i)^\top.\]
We aim to solve the nearest Kronecker approximation problem:
\[\|F_t - L_t^{(1)} \otimes L_t^{(2)} \otimes \dots \otimes L_t^{(d)}\|\to \min_{L_t^{(1)}, \dots, L_t^{(d)}},\]
where $L_t^{(k)} \in \mathbb{R}^{n_k\times n_k}$, using projector-splitting. 

Let $G_t^{(k)} \in \mathbb{R}^{n_k\times (n_1\dots n_{k-1}n_{k+1}\dots n_d)}$ denote the unfolding (rearrangement) of tensor $G_t$ into a matrix by the $k$-th dimension.

The following algorithm is an adaptation of projector-splitting \cite[Section 4]{ceruti2022rank} for the nearest Kronecker approximation of Fisher matrix in case of tensor-shaped parameters. 
\begin{algorithm}
  \caption{
  \texttt{kron\_proj\_split} for tensors. Dynamical Kronecker approximation of Fisher matrix.
  }\label{alg:kron_proj_split_tensor}
\begin{algorithmic}[1]
	\Require 
    $\{L^{(i)}_t\}_{i=1}^{d} \in \mathbb{R}^{n_i \times n_i}$ -- Kronecker product approximation of the Fisher matrix on the optimization step~$t$, $G_t \in \mathbb{R}^{n_1 \times n_2\times \dots \times n_d}$ -- gradient with respect to the layer $W_t$.
    \For{$k=1\dots d$}
    \State $\mathrm{norm}_k = \left\|{L^{(1)}_t}\right\|_{{F}}\dots \left\|{L^{(k-1)}_t}\right\|_{{F}}\left\|{L^{(k+1)}_t}\right\|_{{F}}\dots \left\|{L^{(d)}_t}\right\|_{{F}}$
	\State $\hat{L}^{(k)}_t = \displaystyle L^{(k)}_t \mathrm{norm}_k^2 + G^{(k)}_t \left(L_t^{(1)} \otimes \dots \otimes L_t^{(k-1)} \otimes  L_t^{(k+1)}\otimes  \dots \otimes L_t^{(d)} \right) G^{(k)\top}_t$

	\State $L^{(k)}_{t+1} = \hat{L}^{(k)}_t/  \|{\hat{L}^{(k)}_t}\|_{{F}}$\EndFor
   
	\State $S = \prod_k\displaystyle \langle {L^{(k)}_t}, {L^{(k)}_{t+1}}\rangle_{{F}}+ \mathrm{vec}(G_t)\left(L_{t+1}^{(1)} \otimes  \dots \otimes L_{t+1}^{(d)} \right)  \mathrm{vec}(G_t)^\top$
    \State \Return $\{S^{1/d}L^{(k)}_{t+1}\}_{i=1}^{d}.$

\end{algorithmic}
\end{algorithm}

\end{document}

%% file: math_commands.tex
\usepackage{amsmath,amsfonts,bm}

\def\eqref#1{equation~\ref{#1}}
\def\Eqref#1{Equation~\ref{#1}}

\def\Algref#1{Algorithm~\ref{#1}}

\def\1{\bm{1}}

\DeclareMathAlphabet{\mathsfit}{\encodingdefault}{\sfdefault}{m}{sl}
\SetMathAlphabet{\mathsfit}{bold}{\encodingdefault}{\sfdefault}{bx}{n}

\newcommand{\comment}[2]{\Statex \hspace{#2} $\triangleright$ #1}